\documentclass[10pt]{article}
\usepackage[margin=0.88in]{geometry}
\usepackage[T1]{fontenc}
\usepackage[utf8]{inputenc}
\usepackage{lmodern,microtype}
\usepackage{amsmath,amssymb,amsthm}
\usepackage{graphicx,booktabs,tabularx,array,xcolor,float}
\usepackage[font=small,labelfont=bf]{caption}
\usepackage[numbers,sort&compress]{natbib}
\usepackage[colorlinks=true,allcolors=blue!45!black]{hyperref}
\usepackage{placeins}
\newtheorem{proposition}{Proposition}

\newcommand{\method}{Geo-vision}
\newcommand{\R}{\mathbb{R}}
\newcommand{\pos}[1]{\left[#1\right]_{+}}
\newcommand{\readas}[1]{\par\noindent\textit{Read as:} #1\par}
\newcolumntype{Y}{>{\raggedright\arraybackslash}X}
\newcommand{\GeoAP}{0.7577}
\newcommand{\GeoAUC}{0.9024}
\newcommand{\YoloAP}{0.7458}
\newcommand{\AppearanceGain}{0.0132}
\newcommand{\TopoGain}{0.0049}
\newcommand{\YoloGain}{0.0119}
\newcommand{\TopoCI}{[+0.0018, +0.0081]}
\newcommand{\TopoAUCGain}{0.0003}
\newcommand{\LowBudgetGain}{0.0448}
\newcommand{\TrainSeventyFive}{33,179}
\newcommand{\TrainFull}{44,242}
\newcommand{\APSeventyFive}{0.7496}
\newcommand{\LearningSeventyFiveCI}{[-0.0031, +0.0109]}
\newcommand{\GeoLatency}{26.98}

\newcommand{\YoloLatency}{11.20}

\newcommand{\GeoMemory}{185}
\newcommand{\YoloMemory}{437}

\newcommand{\FOnePrecision}{0.6581}
\newcommand{\FOneSensitivity}{0.7003}
\newcommand{\FOneSpecificity}{0.9078}
\newcommand{\FOneScore}{0.6786}
\newcommand{\HighPrecision}{0.9662}
\newcommand{\HighPrecisionRecall}{0.1356}
\newcommand{\SmallTP}{6}
\newcommand{\SmallAP}{0.2119}
\newcommand{\LargeAP}{0.7659}
\newcommand{\PartLowAP}{0.0374}
\newcommand{\PartHighAP}{0.9426}
\newcommand{\PartLowName}{roof}
\newcommand{\PartHighName}{front bumper}
\hypersetup{pdftitle={Geo-vision: Native-Resolution Geometry and Sampled Topology for Vehicle-Part Damage Classification},pdfauthor={Usef Faghihi, Amir Saki, Serge Jonathan Kogbeti-Wongo}}
\title{\method: Native-Resolution Geometry and Sampled Topology for Vehicle-Part Damage Classification}
\author{Usef Faghihi \quad Amir Saki \quad Serge Jonathan Kogbeti-Wongo\\
\small Universit\'e du Qu\'ebec \`a Trois-Rivi\`eres, Canada\\
\scriptsize\texttt{usef.faghihi@uqtr.ca} \quad\texttt{amir.saki@uqtr.ca}\\
\scriptsize\texttt{Jonathan.Serge.Kogbeti.Wongo@uqtr.ca}}
\date{}
\begin{document}
\maketitle
\begin{abstract}
Geo-vision classifies vehicle-part damage using a fixed representation combining appearance, multiscale geometry, orientation, and sampled digital topology. Connected components, holes, Euler characteristics, and spatial arrangement complement amplitude-sensitive pooling and a supervised tree classifier. We establish amplitude preservation, axial sign invariance, and conditional operator stability. Retrospective checkpoint evaluation on 55,303 crops from 4,916 source groups yields average precision 0.7577, exceeding 0.7458 for an ImageNet-pretrained frozen YOLO11n comparator on identical crops. The complete added structural block, comprising topology, occupancy, and component location, improves average precision by 0.00493 beyond appearance and geometry. Geo-vision uses no external encoder pretraining and has higher AP point estimates at every matched label budget. With approximately 25\% fewer task labels, its AP point estimate exceeds full-data YOLO, although the adjusted interval includes zero. Complete serial CPU inference requires 26.98 ms per crop versus YOLO's 11.20 ms, with peak memory 185 versus 437 MiB. Independent external validation remains necessary.
\end{abstract}

\section{Introduction}
An image of a damaged vehicle part may contain a narrow scratch, a fragmented boundary, a dent-induced intensity ridge, or an irregular network of strong responses. Similar average intensity or edge occupancy can accompany different structural arrangements. A representation that averages away these distinctions can lose information relevant to classification. Conversely, connectivity alone is insufficient: a reflection or manufactured panel boundary may have the same topological summary as a damage pattern. \method\ therefore combines photometric, geometric, and topological observations of the same crop.

The task considered here is binary classification of an already localized vehicle part. A colour crop is mapped to a continuous damage score. Localization, damage segmentation, severity, and repair cost are separate tasks. CarDD, for example, provides annotations for damage detection and segmentation \citep{wang2023cardd}; its detection metrics cannot be identified with the crop-level AP used in this study. The present private dataset supplies crop labels and recorded source groups, enabling grouped comparison of image representations while holding the localization task fixed.

The defining idea of \method\ is an explicit structural image representation. Derivatives are observed at several scales before the native branch is enlarged for spatial pooling. Response magnitude is retained alongside axial orientation and cross-scale agreement. Superlevel masks then express how coherent responses connect and enclose holes. A fixed collection of local and global summaries becomes the input to a learned tree classifier. Thus the representation has no externally learned image encoder, while the readout is supervised and depends on task labels.

The code identifies appearance, native magnitudes, orientation, and topology as B, N, O, and T; their concatenation is internally named BNOT. These identifiers describe \method's components and ablations. These coordinate names identify one Geo-vision representation.

We contribute a code-consistent account of \method's structural observations and complementary adjacency conventions; proofs of operator properties with explicit stability conditions; and full-cohort checkpoint reproduction with paired uncertainty, label-budget comparisons, and repeated CPU timing.

\section{Relationship to geometric and topological vision}
Histograms of oriented gradients retain directional organization \citep{dalal2005}; local binary patterns encode texture \citep{ojala2002}; structure tensors estimate an unoriented local axis \citep{bigun1991}; and multiscale Hessians describe ridge-like curvature \citep{frangi1998}. \method\ combines these established operators with native-amplitude observations and sampled topology.

Digital topology describes foreground components and background holes under compatible adjacency conventions \citep{rosenfeld1979}. These quantities complement occupancy: equal foreground area does not determine connectedness or enclosed loops. Persistent homology additionally tracks homology classes through a filtration, and persistence-diagram stability can hold under appropriate assumptions \citep{cohensteiner2007}. The released \method\ code instead samples Betti/Euler measurements at finitely many response thresholds. It computes no persistence diagram, no birth--death pairing, and no three-dimensional homology. The term \emph{cross-scale survival} below refers to coherent response strength across image scales, not to a persistence lifetime.

The readout uses LightGBM gradient-boosted trees \citep{ke2017}. The comparator uses a frozen ImageNet-pretrained YOLO11n classification encoder \citep{ultralytics2024}. Different preprocessing and readout families make this a comparison of complete crop classifiers, not an isolated encoder comparison or an evaluation of a fine-tuned detector.

\section{The Geo-vision representation}
\subsection{Input, preprocessing, and appearance}
Let $I\in\{0,\ldots,255\}^{H\times W\times3}$ be a nonempty colour crop. The implementation accepts a uint8 BGR array whose aspect-preserving texture resize retains at least five pixels on each axis; narrower crops are rejected explicitly. A fixed preprocessing rule proposes narrow orange strokes using colour, thickness, elongation, and colour-uniformity criteria; selected pixels are replaced by averages of nearby unselected pixels. Selection uses no outcomes. Genuine image content satisfying the same rule can also be changed, so marker cleaning is part of the measured pipeline rather than a guaranteed removal of annotation artefacts.

The appearance block $B\in\R^{2666}$ combines a 1,536-coordinate spatial representation with 1,130 texture/colour coordinates. Spatial measurements include local gradient statistics, orientation histograms, absolute and relative edge occupancy, component summaries, and elementary topology. Texture/colour measurements include local binary patterns, intensity and locally standardized histograms, Lab layout, and a short-wave Gabor bank. Resizing preserves aspect ratio; the spatial branch has an adaptive long side between 160 and 256 pixels, and the texture branch between 144 and 224. Gabor responses use a bounded 96-pixel-long copy. These resizing rules are distinct from the native derivative branch.

For that branch, the cleaned crop is downsampled only when its longer side exceeds 256 pixels. It is never enlarged before differentiation. Its normalized grayscale observation and median observation are
\begin{equation}\label{eq:gray}
u=(0.299I_R+0.587I_G+0.114I_B)/255,
\qquad \widetilde u=\operatorname{median}_{3\times3}(u).
\end{equation}
\readas{observe the same crop in raw and median-filtered grayscale. Their agreement and disagreement describe different structures; neither observation is assumed to isolate physical damage.}

\begin{table}[tb]\centering\small
\caption{The fixed coordinate schema of \method. The final vector is float32 and rounded to six decimal places.}\label{tab:schema}
\begin{tabularx}{\linewidth}{l r Y}\toprule
Observation block & Coordinates & What is represented\\\midrule
Appearance, $B$ & 2,666 & Spatial gradients, texture, colour, occupancy, and elementary topology\\
Native magnitudes, $N$ & 992 & Absolute strengths and raw/median multiscale summaries\\
Orientation/ridges, $O$ & 480 & Axial direction, scale agreement, and signed curvature\\
Multiscale topology, $T$ & 776 & Occupancy, tile Euler values, component locations, and global Betti counts\\\midrule
\method, $\Phi(I)$ & 4,914 & Concatenation supplied to a supervised tree readout\\\bottomrule
\end{tabularx}
\end{table}

\subsection{Native multiscale geometry}
For $\sigma\in\{0.7,1.4,2.8\}$ pixels, let $G_\sigma$ denote Gaussian smoothing. Scale-normalized derivatives, gradient magnitudes, and structure tensors are
\begin{equation}\label{eq:derivative}
u_\sigma=G_\sigma*u,\quad g_\sigma=\sigma\nabla u_\sigma,
\quad m_\sigma=\|g_\sigma\|_2,
\quad J_\sigma=G_{1.5\sigma}*(g_\sigma g_\sigma^\top).
\end{equation}
\readas{measure how strongly intensity changes at each scale, then average its local directional organization. The factor $\sigma$ expresses strength in scale-normalized units.}
The code uses $3\times3$ Sobel first derivatives scaled by $1/8$ and reflected image boundaries. Writing the tensor entries as $J_{11},J_{12},J_{22}$, define
\begin{equation}\label{eq:axial}
d_\sigma=\sqrt{(J_{11}-J_{22})^2+4J_{12}^2},\quad
s_\sigma=\sqrt{d_\sigma},\quad
z_\sigma=\frac{(J_{11}-J_{22},2J_{12})}{d_\sigma}.
\end{equation}
\readas{retain an anisotropic response strength and a doubled-angle orientation. The doubled angle identifies an axis even when its direction is reversed.}
The implementation sets $z_\sigma=0$ when $d_\sigma\leq10^{-12}$. Otherwise it is a unit vector $(\cos2\theta,\sin2\theta)$. For adjacent scales, the coherent response is
\begin{equation}\label{eq:coherence}
c_k=\pos{z_{\sigma_k}^{\top}z_{\sigma_{k+1}}}
       \min(s_{\sigma_k},s_{\sigma_{k+1}}),\qquad k\in\{1,2\}.
\end{equation}
\readas{a response survives only as strongly as the weaker adjacent-scale observation, discounted when their axes disagree. It is a local strength field, not a causal effect or a persistence diagram.}

For a nonnegative field $f$ and normalized nonnegative cell weights $w_i$, define
\begin{equation}\label{eq:pool}
M(f)=\sum_iw_if_i,\qquad
P_4(f)=\left(\sum_iw_if_i^4\right)^{1/4},\qquad\sum_iw_i=1.
\end{equation}
\readas{the mean records distributed response strength; the fourth-order mean emphasizes stronger responses while preserving their amplitude.}
Area pooling supplies the weights. A field smaller than the requested grid is bilinearly enlarged first; its fourth power is computed \emph{after} this interpolation. No per-cell normalization removes amplitude in the native magnitude/orientation branch.

The six raw/median gradient-magnitude fields are pooled on $4\times4$ grids. Five further fields use both $4\times4$ and $8\times8$ grids: minimum magnitude over the three raw scales; the analogous median minimum; their pointwise minimum; their positive raw-minus-median difference; and $|u-\widetilde u|$. Thus $N$ has $6(2\cdot16)+5[2(16+64)]=992$ coordinates.

Eight axial response fields use $\theta_j=j\pi/8$ and
\begin{equation}\label{eq:bins}
a_j=\max_{k=1,2}\left\{c_k\pos{z_{\sigma_k}^{\top}
            (\cos2\theta_j,\sin2\theta_j)}\right\},\quad j=0,\ldots,7.
\end{equation}
\readas{assign coherent strength to fixed image-axis bins, retaining magnitude and spatial position rather than only a normalized orientation histogram.}
Raw, median, and pointwise raw/median coherent fields also produce three scale-disagreement summaries, $D(c)=\max(c_1,c_2)-(c_1+c_2)/2$, and two raw/median intersection fields.

Signed ridge fields use the scale-normalized Hessian $H_\sigma=\sigma^2\nabla^2u_\sigma$, with eigenvalues $\lambda_+\geq\lambda_-$:
\begin{equation}\label{eq:ridge}
r_\sigma^+=\pos{\lambda_+-|\lambda_-|},\qquad
r_\sigma^-=\pos{-\lambda_--|\lambda_+|}.
\end{equation}
\readas{keep complementary signs of dominant local curvature, so that bright and dark ridge-like structures can be observed separately.}
Second derivatives use Sobel scaling $1/4$. Hessian doubled-angle agreement weights the weaker adjacent-scale ridge; the minimum of the two adjacent-pair responses is retained per polarity. All 15 orientation/ridge fields use mean and fourth-order pooling on $4\times4$ grids, yielding $O\in\R^{480}$.

\subsection{Sampled digital topology}
For each raw coherent field and each fixed level $\tau\in\{0.01,0.03\}$, form the superlevel mask
\begin{equation}\label{eq:masks}
A_{k,\tau}(x)=\mathbf{1}\{c_k(x)\geq\tau\}.
\end{equation}
\readas{a pixel belongs to the structural foreground when its coherent strength reaches the chosen level. The two levels produce nested masks for each field.}
Foreground uses 8-connectivity; its complementary background uses 4-connectivity. Define $\beta_0$ as the number of foreground components and $\beta_1$ as the number of background components not touching the image border. With these conventions,
\begin{equation}\label{eq:betti}
\chi(A)=\beta_0(A)-\beta_1(A).
\end{equation}
\readas{the Euler characteristic balances fragmentation against enclosed holes. A filled component has $(\beta_0,\beta_1)=(1,0)$; a ring has $(1,1)$.}
Complementary adjacency matters for diagonal contacts: choosing the same connectivity indiscriminately for foreground and background can make the intended interpretation inconsistent.

Each of the four masks contributes 64 occupancy values, 64 independent tile Euler characteristics, 64 log-transformed component-centroid counts, and the global values $\log(1+\beta_0)$ and $\log(1+\beta_1)$. Components are identified on the full mask before their centroids are assigned to $8\times8$ cells. Tile Euler values are measured with each tile's own exterior background. Their sum is generally \emph{not} the full-image Euler characteristic, because tile boundaries cut structures. The topology block therefore has $4(3\cdot64+2)=776$ coordinates.

This design retains both topology and position. It distinguishes a connected response from fragmented responses with the same total area, and a filled structure from one enclosing a hole. It does not uniquely recover shape: many different images share the same Betti numbers, and even the complete finite vector is not an injective representation of all images. Across thresholds, counts are sampled without tracking the identities of individual homology classes. Statements about a formal persistence lifetime would therefore exceed what the code computes.

\subsection{Supervised readout and ablations}
The complete feature map and damage score can be written as
\begin{equation}\label{eq:score}
\Phi(I)=\operatorname{round}_6\bigl(B(I),N(I),O(I),T(I)\bigr),
\qquad p(I)=\operatorname{logit}^{-1}\!\left(\sum_{t=1}^{600}h_t(\Phi(I))\right),
\end{equation}
\readas{concatenate the fixed observations, then add the learned leaf contributions of 600 binary-objective trees and convert their raw score to a probability.}
The trees learn from supplied damage labels; the image operators and their thresholds remain fixed. We compare appearance alone ($B$), appearance plus geometry ($B+N+O$), and complete \method. Since $B$ already contains elementary topology, removing $T$ tests the \emph{additional multiscale structural block}, not the absence of all topology. Moreover, $T$ contains occupancy and component-location measurements alongside Euler/Betti values; its ablation effect cannot be attributed to homology alone. A topology-only performance claim requires a separate Euler/Betti ablation controlling for occupancy and centroid location.

\section{Mathematical properties and their scope}\label{sec:properties}
The following results concern the defined operators in real arithmetic. Floating-point computation and six-decimal quantization are subsequently checked by replay. They do not assert that every physical damage pattern is identifiable from a crop.

\begin{proposition}[Amplitude retention and sparse response]\label{prop:pool}
For nonnegative fields on a fixed grid, $M(af)=aM(f)$ and $P_4(af)=aP_4(f)$ for $a\geq0$. Both operators are 1-Lipschitz in the supremum norm. If $f=A>0$ on cells of total weight $p>0$ and zero elsewhere, then
\begin{equation}\label{eq:sparse}
M(f)=Ap,\qquad P_4(f)=Ap^{1/4},\qquad P_4(f)/M(f)=p^{-3/4}.
\end{equation}
\end{proposition}
\readas{a spatially sparse strong response contributes more to the fourth-order mean than to the ordinary mean. This is an amplitude-sensitive pooling property, not a proof of population rarity detection.}
\begin{proof}
Homogeneity follows by factoring $a$ out of Eq.~\eqref{eq:pool}. For the mean, $|M(f)-M(g)|\leq\sum_iw_i|f_i-g_i|\leq\|f-g\|_\infty$. For the fourth-order mean, the reverse triangle inequality for the weighted fourth norm gives $|P_4(f)-P_4(g)|\leq P_4(|f-g|)\leq\|f-g\|_\infty$. Substitution of the two-level field gives Eq.~\eqref{eq:sparse}.
\end{proof}
For one active element among 64 equally weighted elements, the ratio is $64^{3/4}\approx22.63$. Both statistics still decrease as support becomes smaller, and neither creates information absent from the image. Unlike normalizing a nonzero vector by its own norm, the pooling preserves positive rescaling. This sensitivity can retain a weak-versus-strong distinction, but it can also respond to illumination and contrast changes. The code implements no density estimator, escort distribution, or causal intervention operator.

\begin{proposition}[Axial sign invariance and ideal contrast scaling]\label{prop:axial}
The tensor and doubled-angle orientation are unchanged by $g_\sigma\mapsto-g_\sigma$. If a fixed grayscale observation becomes $u'=au+b$ with $a>0$, linear derivative/smoothing operators give $g_\sigma'=ag_\sigma$, $J_\sigma'=a^2J_\sigma$, and $s_\sigma'=as_\sigma$. Wherever both observations satisfy the orientation-validity condition, $z_\sigma'=z_\sigma$ and $c_k'=ac_k$.
\end{proposition}
\begin{proof}
The outer product $gg^\top$ is unchanged by sign reversal. A constant $b$ is annihilated by differentiation. The entries of $J$ scale by $a^2$, so $d$ scales by $a^2$, its square root by $a$, and the quotient defining $z$ cancels the factor $a^2$. Equation~\eqref{eq:coherence} then preserves orientation agreement and scales the minimum strength by $a$.
\end{proof}
The condition excludes a contrast change that crosses the numerical orientation floor. Quantization, clipping, marker selection, and fixed absolute mask thresholds further prevent an unconditional contrast-invariance claim for the complete pipeline. Axial sign invariance is also not rotation invariance: fixed image-axis bins and spatial cells change under an image rotation. Retaining position and strength is an intentional representational choice.

\begin{proposition}[Conditional topological stability]\label{prop:margin}
For a finite field $f$ and threshold $\tau$, let $\gamma=\min_x|f(x)-\tau|>0$. If $\|f-g\|_\infty<\gamma$, the masks $\mathbf{1}\{f\geq\tau\}$ and $\mathbf{1}\{g\geq\tau\}$ coincide. Hence their occupancy, component locations, Betti counts, and tile Euler values coincide. For several fields and thresholds, take the minimum of their positive margins.
\end{proposition}
\begin{proof}
The perturbation is smaller than every pixel's distance to the threshold, so no pixel can cross it. All subsequent mask-dependent calculations receive the same binary input.
\end{proof}
\readas{topology is unchanged when every response stays on its original side of each threshold. The bound is on the response field; an image-level guarantee additionally needs a bound from image perturbations to that field.}
Without a positive margin, arbitrarily small perturbations at a bridge can merge components, and at a contour break can remove a hole. A formal persistence-diagram stability theorem \citep{cohensteiner2007} does not make fixed-threshold Betti counts unconditionally continuous. The released pipeline uses sampled counts, so its guarantee must be stated at that level.

\begin{proposition}[Conditional preservation of a tree score]\label{prop:trees}
Let a fitted tree ensemble receive $v\in\R^D$. For every internal node on the paths traversed by $v$, write its feature index and threshold as $(j,s)$. If another vector $v'$ obeys $|v_j'-v_j|<|v_j-s|$ at every such node, and all these margins are positive, it reaches the same leaves and receives the same ensemble score.
\end{proposition}
\begin{proof}
At the root, the strict bound preserves the side of the split. Inductively it preserves the side of each subsequent node on the original path. Each leaf contribution is therefore unchanged, as is their sum and its logistic transformation.
\end{proof}
This explains how a scheduling optimization can preserve predictions despite small floating-point feature differences. It is a conditional statement, not a global Lipschitz bound for trees or an adversarial-robustness certificate. The empirical replay checks directly compare the evaluated scores, without assuming that all possible images have positive split margins.

\section{Computation and controlled structural checks}\label{sec:controlled}
With a fixed number of scales, orientations, thresholds, and grids, native filtering, pooling, and connected-component scans have linear cost in the number $P$ of working pixels. Gaussian/Sobel supports are fixed by the declared scales. A planar mask requires components and complementary background scans, rather than construction of high-dimensional simplicial complexes. A readout traversing $K$ trees of maximum depth $L$ costs $O(KL)$; the deployed descriptor dimension is fixed. These observations explain a bounded and auditable computational pathway, but do not by themselves establish a speed ranking against another implementation.

The optimized source batches independently padded tile Euler calculations using integer $2\times2$ stencils. Zero rows and columns separate tiles so the batching preserves each tile's exterior convention. It also caches immutable schemas and omits unused pooling work. Integer Euler counts can consequently be checked exactly. Floating-point reductions are checked separately because different scheduling can move a rounded coordinate at a quantization boundary. The descriptor's physical interpretation and coordinate order remain fixed.

\begin{figure}[tb]\centering
\includegraphics[width=0.86\linewidth]{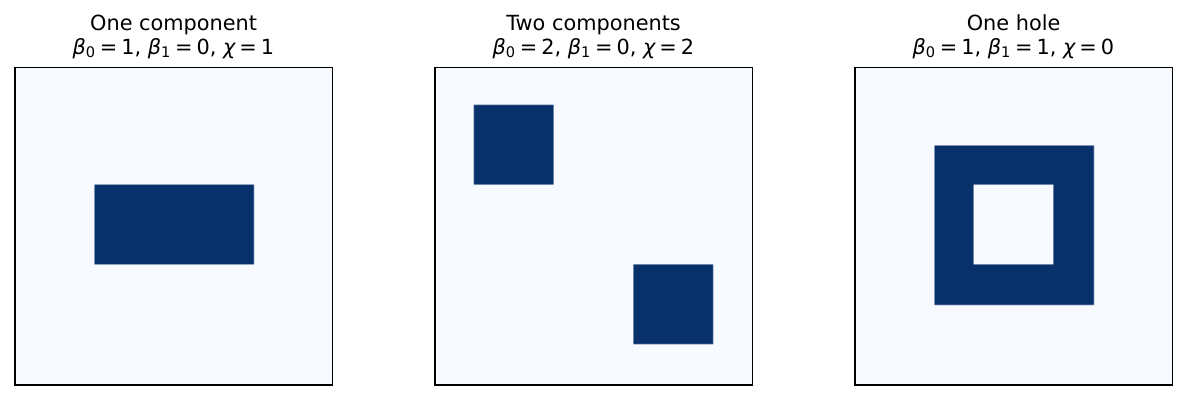}
\caption{Controlled topology checks re-executed with the released operators. The first two masks have equal area (32 pixels) but one versus two foreground components. The ring encloses one hole. This demonstrates information complementary to total occupancy; it does not establish damage-classification accuracy.}\label{fig:topology}
\end{figure}

Every one of the 512 binary $3\times3$ masks is checked against an independent foreground/background component calculation. The test suite additionally covers thin masks, empty tiles, irregular tile boundaries, feature order, small-field interpolation before the fourth power, tied prediction scores, and group-leakage rejection. The evaluated snapshot passed 16 repository tests; the revised source package passes 46, including input, cache, and training-subset regression checks. In a two-level ring field with foreground 0.8, background 0.2, and threshold 0.5, all 1,000 uniform perturbations bounded by 0.299 preserve the mask, as Proposition~\ref{prop:margin} predicts. Conversely, changing a bridge from $0.5-10^{-8}$ to $0.5+10^{-8}$ merges two components. Across 81 positive contrast multipliers, mean and fourth-order strengths change with amplitude while a normalized two-coordinate vector is constant to numerical precision.

These checks test the measurement semantics and the conditions of the propositions. They are not an external visual benchmark. In particular, an equal-area construction does not imply that appearance descriptors always fail or that topology alone recognizes rare damage.
\section{Experimental design and reproducibility}\label{sec:design}
\subsection{Original cohort and immutable grouping}
The corpus combines vehicle images belonging to Progi, a company active in automotive insurance, with images downloaded from the Internet. Image annotations were performed by Serge Jonathan Kogbeti-Wongo. The original image-selection criteria, collection period, and detailed annotation protocol remain to be documented.

The original five image shards contain 69,108 vehicle-part crops. Byte verification and the authenticated cohort parser recover 55,303 eligible development crops from 4,916 recorded source groups and 16 part categories. Eight duplicate-piece annotations are excluded, and 13,797 historically reserved images are excluded from fitting and evaluation. The accounting is exact: $55{,}303+8+13{,}797=69{,}108$. Eligible labels contain 11,185 damaged and 44,118 undamaged crops, a damage prevalence of 20.22\%.

The five saved outer folds are retained without resplitting. They contain 10,882--11,173 crops each, and a recorded source group belongs to exactly one fold. Eligible images have no byte-identical duplicates, and recorded validation groups do not cross folds. Source-group identifiers describe the supplied grouping; the records do not independently establish one group per unique vehicle or acquisition session. The median shorter and longer crop sides are 77 and 155 pixels, respectively; 7,922 crops have a shorter side below 32 pixels.

The feature extractor receives image arrays, not filenames, source identifiers, part labels, damage labels, or explicit dimension variables. Neutral row IDs bind feature order to the saved cohort. The parser verifies the original assignment, exclusions, row ordering, and fold artifacts by cryptographic hashes. Labels are joined only for eligible rows in the fitting/evaluation stage. This design limits leakage through recorded metadata; it cannot resolve unrecorded dependencies or annotation errors.

\subsection{Fitted models and checkpoint reproduction}
The tree candidates share a binary LightGBM objective with 600 trees, learning rate 0.03, 63 leaves, minimum 30 observations per leaf, 63 histogram bins, 75\% feature subsampling per tree, and $L_1/L_2$ penalties 0.1/4.0. Row subsampling is disabled. Deterministic column-wise construction is requested; the outer seed is $20260818+f-1$ for fold $f$. The configurations are fixed across folds.

The comparator uses official ImageNet-pretrained YOLO11n classification weights, authenticated by SHA-256. Uncleaned RGB images are aspect-preserving bilinear letterboxes on a $224\times224$ canvas with padding intensity 114 and division by 255. Its frozen 1,280-coordinate pooled embedding is followed by a training-fold-fitted standard scaler and class-balanced logistic regression ($C=1$, L-BFGS, tolerance $10^{-7}$, at most 2,000 iterations). The initial fits check convergence. Thus both pipelines learn their task readouts from the same rows, but only the comparator uses external image-encoder pretraining.

For this paper, the complete eligible image cohort is reprocessed using repository commit \texttt{53f97847}. The saved group-held-out model checkpoints are loaded and replayed on regenerated features. All 20 full-data outer readouts and all 30 reduced-data readouts are checked against their saved predictions; image-to-score timings are remeasured. The reference descriptor at commit \texttt{a3afa719} is also re-extracted to check the optimization. This is \emph{checkpoint reproduction}: fitted parameters and previously selected inner-fold thresholds are retained, not re-estimated. Fresh feature computation and repeated evaluation verify the deployed calculation; they do not constitute a new independent cohort or a fresh refitting experiment.

The complete model receives fold-specific operating thresholds selected from three stratified grouped inner folds inside each outer training set. The primary rule maximizes inner F1, taking the higher threshold on a tie. The secondary rule maximizes recall among inner thresholds with precision at least 0.95, breaking ties by precision and then threshold. An infeasible rule predicts no positives. These fixed thresholds transfer to the corresponding held-out groups; the inner precision target is not a held-out guarantee.

\subsection{Endpoints and paired uncertainty}
Average precision is the primary ranking endpoint; receiver operating characteristic area (ROC-AUC) is secondary \citep{davis2006}. AP is the non-interpolated sum of precision weighted by recall increments. Pooled predictions include exactly one group-held-out score for every eligible crop and candidate. Because pooling assumes cross-fold score comparability, foldwise AP is also reported. Operating-point precision, recall, specificity, F1, and confusion counts remain separate from AP.

Ten thousand common bootstrap draws resample recorded source groups within each outer fold, retaining all crops and multiplicities of each sampled group. Candidates share identical weights in every draw. Metric intervals use percentile 95\% bounds. Three fixed AP contrasts, complete \method\ minus appearance, appearance plus geometry, and frozen YOLO, use Bonferroni-adjusted 98.33\% intervals. These intervals describe group-sampling variation conditional on the fitted predictions; they exclude refitting and model-development uncertainty.

The label-budget study uses nested 25\%, 50\%, and 75\% subsets of outer-training source groups, with identical selected rows for both readouts and unchanged test folds. Each fraction involves five outer fits per method, hence 30 reduced-data checkpoints. All settings are fixed, and the 100\% point uses the full-data fits. The three comparisons against full-data YOLO receive the same simultaneous-interval adjustment. Because this is an already-used development cohort and one nested subset schedule, the learning curves are exploratory.

\subsection{Repeated CPU timing}
The same outcome-independent sample of 256 eligible crops is timed at batch size one on one pinned CPU core with one thread per method. All six method orders balance reference \method, optimized \method, and frozen YOLO. Each run prewarms the filesystem cache and excludes 16 warm-up predictions. Timing includes file reading, decoding, preprocessing, image representation, and classification; process/model initialization and training are excluded. Peak process memory includes imports and model loading. Crop-bootstrap speed-ratio intervals quantify this sampled workload, not variation over processors or GPU deployments.
\section{Results}\label{sec:results}
\subsection{Structural information and damage ranking}
\begin{table}[tb]\centering\small
\caption{Group-held-out ranking from newly replayed checkpoints.}\label{tab:ranking}
\begin{tabular}{lrrr}\toprule
Pipeline & AP & 95\% interval & ROC-AUC\\\midrule
Appearance & 0.744463 & [0.7354, 0.7535] & 0.896883\\
Appearance + geometry & 0.752745 & [0.7436, 0.7619] & 0.902127\\
Geo-vision & 0.757674 & [0.7487, 0.7665] & 0.902403\\
Frozen YOLO11n & 0.745805 & [0.7374, 0.7542] & 0.894295\\
\bottomrule\end{tabular}
\par\smallskip\small Intervals resample recorded source groups and condition on the fitted readouts.
\end{table}
\method\ achieves AP \GeoAP{} and ROC-AUC \GeoAUC{} on the complete group-held-out cohort. The paired AP gains are \AppearanceGain{} over appearance alone, \TopoGain{} over appearance plus geometry, and \YoloGain{} over frozen YOLO. The simultaneous interval for the added structural-block contrast is \TopoCI{}. Thus the additional structural block improves damage ranking beyond the geometric representation in this fitted development evaluation. The ablation concerns the whole block, which includes occupancy and component position as well as Euler/Betti values.

\begin{figure}[tb]\centering
\includegraphics[width=0.97\linewidth]{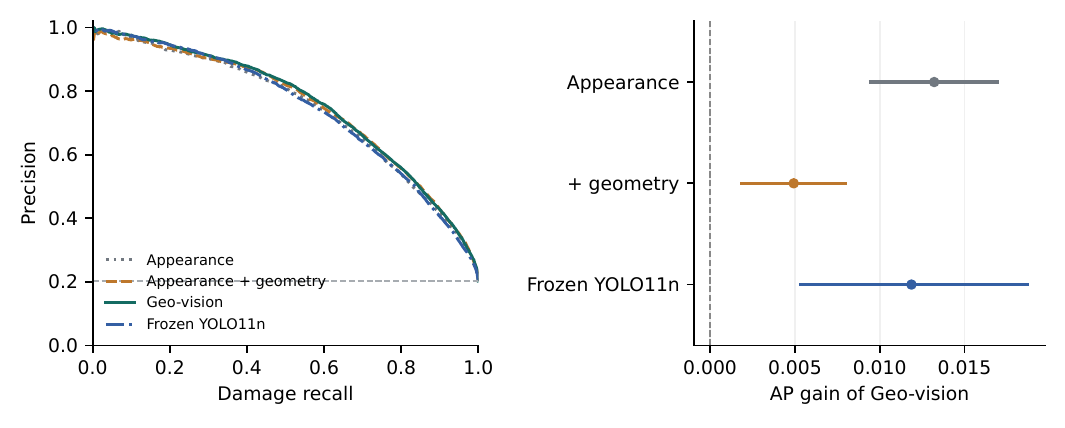}
\caption{Fresh checkpoint evaluation on the same 55,303 group-held-out crops. Left: precision--recall curves, with damage prevalence shown by the horizontal dashed line. Right: complete \method\ minus each comparator, with adjusted 98.33\% AP intervals. All candidates use the same group-bootstrap weights.}\label{fig:classification}
\end{figure}

\begin{table}[tb]\centering\small
\caption{AP within each immutable held-out fold.}\label{tab:folds}
\begin{tabular}{rrrrr}\toprule
Fold & Appearance & + geometry & Geo-vision & Frozen YOLO\\\midrule
1 & 0.7599 & 0.7695 & 0.7743 & 0.7653\\
2 & 0.7366 & 0.7443 & 0.7503 & 0.7389\\
3 & 0.7285 & 0.7376 & 0.7441 & 0.7382\\
4 & 0.7551 & 0.7617 & 0.7619 & 0.7463\\
5 & 0.7419 & 0.7513 & 0.7581 & 0.7400\\
\bottomrule\end{tabular}

\end{table}
The foldwise values expose variation that pooled ranking can conceal. The pooled ROC-AUC gain from the added structural block is \TopoAUCGain{}, smaller than the corresponding AP increment. This indicates that the extra representation particularly affects positive-class ranking; it does not establish dominance at every classification threshold. No damage localization metric is inferred from these crop-level scores.

\subsection{External pretraining and task-label budgets}
\method's image representation uses zero external pretraining images. The comparator's authenticated checkpoint identifies ImageNet as its training dataset and contains 1,000 original classes. ImageNet-1k has a documented training set of 1,281,167 images \citep{russakovsky2015,ultralytics2024}; this nominal set size is not an independently audited count of distinct checkpoint exposures. In contrast, full-data task readouts for both pipelines use the same 44,130--44,421 labeled crops per outer fold. A lack of encoder pretraining and a reduction in task-label requirements are distinct properties.

\begin{table}[tb]\centering\small
\caption{Matched task-training budgets for the two readouts.}\label{tab:learning}
\begin{tabular}{rrrr}\toprule
Training groups & Mean crops/fold & Geo-vision AP & Frozen YOLO AP\\\midrule
25\% & 11,101.8 & 0.714900 & 0.670082\\
50\% & 22,123.6 & 0.735524 & 0.720404\\
75\% & 33,179.0 & 0.749582 & 0.737861\\
100\% & 44,242.4 & 0.757674 & 0.745805\\
\bottomrule\end{tabular}
\par\smallskip\small The image operators and classifier settings are fixed. Reduced-data readouts are replayed, not refitted.
\end{table}
At matched label budgets, complete \method\ has higher AP point estimates than the frozen YOLO readout. At the 25\% source-group budget, its AP advantage is \LowBudgetGain{}. Using 75\% of training groups, with a mean of \TrainSeventyFive{} labeled crops per fold, \method\ reaches AP \APSeventyFive{}, compared with \YoloAP{} for full-data YOLO using \TrainFull{} crops per fold. This is a promising point-estimate result with approximately 25\% fewer task labels. Its adjusted difference interval, \LearningSeventyFiveCI{}, includes zero; superiority and equivalence at that smaller budget are not established. The 25\% and 50\% \method\ budgets remain below full-data YOLO in the planned comparisons.

\begin{figure}[tb]\centering
\includegraphics[width=0.88\linewidth]{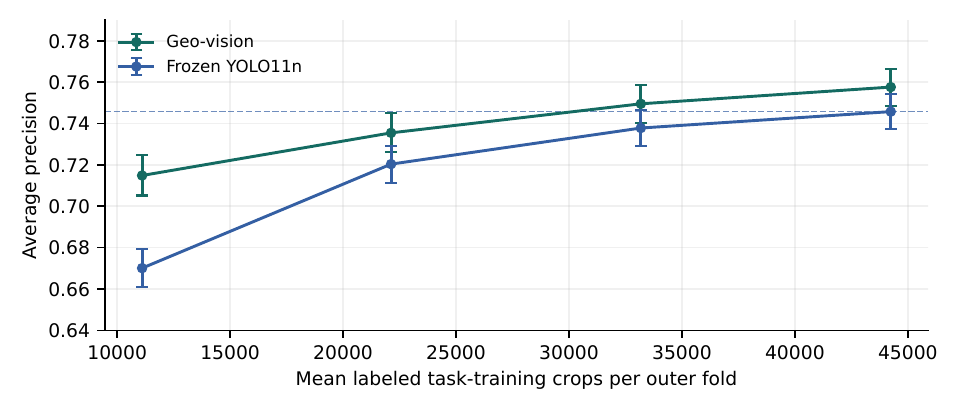}
\caption{Matched, nested training-group budgets with unchanged held-out folds. Both methods receive exactly the same task-training crops at each point. Error bars are conditional 95\% source-group intervals. The dashed line denotes the full-data frozen YOLO point estimate, not an equivalence threshold.}\label{fig:learning}
\end{figure}

\subsection{Inference speed and memory}
\begin{table}[tb]\centering\small
\caption{Complete image-to-score CPU performance on 256 common crops, six balanced orders per environment.}\label{tab:runtime}
\begin{tabular}{llrrr}\toprule
CPU & Pipeline & Mean ms & Crops/s & Peak MiB\\\midrule
EPYC 9V74 & Geo-vision & 22.59 & 44.3 & 227\\
EPYC 9V74 & Frozen YOLO + head & 11.83 & 84.5 & 511\\
Xeon 8573C & Geo-vision & 26.98 & 37.1 & 185\\
Xeon 8573C & Frozen YOLO + head & 11.20 & 89.3 & 437\\
\bottomrule\end{tabular}
\par\smallskip\small One pinned core, one thread, batch size one; initialization and warm-up excluded; reading, preprocessing, features and classification included.
\end{table}
On Intel Xeon Platinum 8573C, Geo-vision processes 37.1 crops per second with mean image-to-score latency \GeoLatency{} ms and peak process memory \GeoMemory{} MiB. Frozen YOLO11n processes 89.3 crops per second at \YoloLatency{} ms and \YoloMemory{} MiB. On AMD EPYC 9V74, the corresponding latencies are 22.59 and 11.83 ms, with 227 and 511 MiB peak memory. Geo-vision therefore uses approximately 56--58\% less peak process memory across these two measured environments, while frozen YOLO has lower serial inference latency. The AP and label-budget advantages reported above coexist with this computational trade-off. Comparisons concern complete CPU pipelines, including image loading and both readouts, rather than encoder-only or GPU throughput.

\subsection{Numerical reproducibility}
Feature reproducibility checks compare 271,758,942 feature coordinates; 2,014 differ, with maximum absolute difference 0.00086899102. Euler, component-location, and Betti coordinates agree exactly. All 165,909 held-out predictions from the 15 full-data tree readouts agree exactly between those two extractions. Including the reduced-data readouts, all 331,818 predictions from the 30 tree checkpoints match between the fresh reference and deployed extractions. Relative to the stored predictions, full Geo-vision and appearance-plus-geometry are exact, but 4 predictions across the appearance-only and reduced-data tree arms differ (maximum score change 0.00136407). The full-data YOLO replay differs at 38 scores by at most 1.44e-15; its reduced-data scores are exact. 
Full Geo-vision scores reproduce the stored full-data predictions exactly. Small differences in other readouts are recorded in the reproduction supplement; the tables and intervals use the freshly computed scores. Reference and deployed implementations produce identical tree scores within the same runtime. This distinguishes same-environment computational equivalence from cross-environment bitwise reproducibility.

\subsection{Operating points and heterogeneity}
\begin{table}[tb]\centering\small
\caption{Geo-vision decisions at fixed grouped-inner-selected thresholds.}\label{tab:operating}
\begin{tabular}{lrrrrrrr}\toprule
Rule & Precision & Recall & Specificity & TP & FP & FN & TN\\\midrule
Maximum inner F1 & 0.6581 & 0.7003 & 0.9078 & 7,833 & 4,069 & 3,352 & 40,049\\
Inner precision target & 0.9662 & 0.1356 & 0.9988 & 1,517 & 53 & 9,668 & 44,065\\
\bottomrule\end{tabular}

\end{table}
The grouped-inner maximum-F1 rule gives precision \FOnePrecision{}, sensitivity \FOneSensitivity{}, specificity \FOneSpecificity{}, and F1 \FOneScore{}. The inner precision-target rule achieves held-out precision \HighPrecision{} with sensitivity \HighPrecisionRecall{}. High precision therefore describes a particular decision rule and its recall, not 0.95 AP or near-complete detection.

Only \SmallTP{} of 216 damaged crops with a shorter side below 32 pixels are detected at the maximum-F1 rule. Their AP is \SmallAP{}, compared with \LargeAP{} among larger crops. Resolution, part category, and damage prevalence vary together, so this subgroup contrast is descriptive. It identifies a practical failure stratum that the pooled AP must not conceal. Across the 16 part categories, AP ranges from \PartLowAP{} for \PartLowName{} to \PartHighAP{} for \PartHighName{}; category-level estimates with few positive examples require cautious interpretation.
\section{Discussion}
\method\ preserves complementary aspects of visible structure without external image-encoder learning. Native derivatives observe strength before enlargement; mean/fourth-order pooling retains amplitude; axial agreement describes directional consistency; signed ridges retain curvature polarity; and sampled topology records connectivity and holes. The paired ablation supports the added structural block's usefulness in the studied task.

Sparse-response retention, axial sign invariance, and margin-based mask/leaf preservation are conditional operator properties. Higher AP and improved point estimates at matched label budgets are empirical properties of the fitted pipelines. Zero external pretraining follows from the source; CPU speed and memory depend on the environment. These different kinds of evidence support distinct claims.

Topology measures the filtered response image, rather than the damaged surface itself. Manufactured edges, reflections, and image artefacts may create components or holes; smooth dents and chromatic damage may not create them. The small-crop result shows that native observation cannot recover missing detail. A future independently acquired evaluation should stratify resolution, part type, and damage mechanism, with matched model-selection budgets and a separately evaluated fine-tuned vision comparator. An ablation separating occupancy, local Euler values, centroid counts, and global Betti counts would clarify which measurements drive the added block's gain. Such a design is a testable next step, not a result inferred from the present bundled ablation.

The dataset is retrospective and currently private. Its sources and annotator are identified above, but the original image-selection criteria, collection period, detailed annotation protocol, and inter-rater agreement remain undocumented in the materials available for this study. Recorded source groups reduce overlap across evaluation folds, but independence across unrecorded vehicles or acquisition sessions has not been independently verified. The representation was developed within this dataset family, and checkpoint replay does not provide an independent external evaluation. The bootstrap intervals condition on fitted models and exclude variability from model refitting and alternative training subsets. These limitations restrict conclusions about generalization and the full uncertainty of the reported results.

\section{Conclusion}
\method\ combines amplitude-sensitive geometry with sampled Betti and Euler measurements to characterize damage-related image structure. On 55,303 crops, it achieves AP \GeoAP{} versus \YoloAP{} for frozen YOLO11n; the complete added structural block, comprising topology, occupancy, and component location, contributes \TopoGain{} AP beyond appearance and geometry, with an adjusted interval entirely above zero. Geo-vision needs no external encoder pretraining and has higher AP at every tested label budget. With approximately 25\% fewer task labels, its AP point estimate exceeds full-data YOLO, although the adjusted interval does not establish superiority at that reduced budget. Serial CPU latency is \GeoLatency{} versus \YoloLatency{} ms per crop, while peak memory is \GeoMemory{} versus \YoloMemory{} MiB: Geo-vision uses less memory, and YOLO is faster in this setting. These results support Geo-vision's structural representation and favorable accuracy--label-budget profile; independent data must establish transfer.

\section*{Code availability}
The code is available at \url{https://github.com/joseffaghihi/Geo-vision}.

\section*{Data availability}
The dataset is currently private. The conditions governing confidentiality and redistribution remain to be confirmed with Progi and the relevant rights holders for the Internet-sourced images. No public release or access upon request is currently promised.

\section*{Author contributions}
Serge Jonathan Kogbeti-Wongo annotated the dataset and contributed to the initial software implementation. Usef Faghihi and Amir Saki developed and refined the current software implementation and contributed equally to this study.

\section*{Use of AI tools}
AI tools were used to assist with software development and language editing. The authors take responsibility for the final manuscript and implementation.

\end{document}